\documentclass{amsart}
\usepackage[english,]{babel}
\usepackage{anyfontsize}
\usepackage[T1]{fontenc}
 \usepackage{mathtools}
 \usepackage{geometry} 
\usepackage{graphicx} 
\usepackage{amsfonts,amsmath,amssymb,amsthm}
\usepackage{hyperref}
\usepackage{mathrsfs}
\usepackage{color}

\newcommand{\R}{\mathbb{R}}

\newtheorem{thm}[]{Theorem}
\title{Quantum Geometry insights in Deep Learning}
\author{Noémie C. Combe}
\address{University of Warsaw, Faculty of Mathematics, Warsaw, Poland}
\email{n.combe@uw.edu.pl}

\begin{document}

\maketitle
\begin{abstract}
    In this paper, we explore the fundamental role of the Monge–Ampère equation in deep learning, particularly in the context of Boltzmann machines and energy-based models. We first review the structure of Boltzmann learning and its relation to free energy minimization. We then establish a connection between optimal transport theory and deep learning, demonstrating how the Monge–Ampère equation governs probability transformations in generative models. Additionally, we provide insights from quantum geometry, showing that the space of covariance matrices arising in the learning process coincides with the Connes–Araki–Haagerup (CAH) cone in von Neumann algebra theory. Furthermore, we introduce an alternative approach based on renormalization group (RG) flow, which, while distinct from the optimal transport perspective, reveals another manifestation of the Monge–Ampère domain in learning dynamics. This dual perspective offers a deeper mathematical understanding of hierarchical feature learning, bridging concepts from statistical mechanics, quantum geometry, and deep learning theory. 
\end{abstract}
\section{Introduction}
The connection between multiscale modeling, Boltzmann machines, and PDEs like the WDVV equation or the Monge–Ampère equation arises naturally in statistical mechanics and deep learning. Recent work has linked optimal transport, Wasserstein distances, and neural networks \cite{C0,C1,C2,C3, CD,DMH,Cu}.  
The Boltzmann machine’s free energy function acts like a potential function for a transport map, much like solutions to the Monge–Ampère equation.  
This means deep learning can be understood geometrically as an optimal transport process.

\, 

The \textbf{Key question} we address:
\begin{itemize}
    \item {\it Can the methods of optimal transport (via the Monge–Ampère equation) provide a rigorous foundation for certain generative modeling techniques? Could this method bring new insights for deep learning?}
\end{itemize}

\, 

This paper aims to explore the fundamental role of the Monge–Ampère equation in deep learning, particularly in the context of Boltzmann machines and energy-based models, while also providing insights from quantum geometry that could be crucial for further developments. Due to our recent research \cite{C0,C1,C2,C3} we show that this learning process can be controlled by an important theory, which is quantum geometry, allowing an entire new panel of tools to use. 

\subsection{Main statements}
\begin{center}
\begin{minipage}{11cm}
  {\bf Theorem 1.}  \textit{ Let $Z\sim \mathcal{N}(0,I_m)$ be a latent Gaussian distribution in $\R^{m}$, and let $T:\R^m \to \R^n$ be a linear transformation parametrized by a matrix $W\in \R^{n\times m}$. Define the transformed random variable: \[X=WZ.\]
Then, the covariance matrix of $X$ follows a Wishart distribution:
$$\Sigma=WW^T\sim  \mathcal{N}_n(m,I).$$
    Furthermore, the space of all such covariance matrices, forms a convex symmetric cone, called the Wishart cone,  consisting of all positive semidefinite matrices of rank at most $min(n,m)$. This cone is a Monge-Amp\`ere domain.}
    \end{minipage}
    \end{center}

      \vskip.2cm

This perspective offers a new way to study energy-based models, Boltzmann machines, and neural network training using the geometry of Wishart matrices and optimal transport theory.

Furthermore, based on the author's recent research (see \cite{C1}) we prove the following: 
\begin{center}
\begin{minipage}{11cm}
  {\bf Theorem 2.}    {\it The space of all covariance matrices 
\[
\Sigma = WW^T
\]
arising in the Boltzmann learning process, where \( W \) is an \( n \times m \) matrix with independent Gaussian columns, forms a convex cone of positive semidefinite matrices. This space is naturally endowed with a \textit{Wishart structure}, as each covariance matrix follows a Wishart distribution:

\[
\Sigma = WW^T \sim W_n(m, I).
\]

Furthermore, this space coincides with the \textit{Connes–Araki–Haagerup (CAH) cone} associated with a finite-dimensional von Neumann algebra. } 
\end{minipage}
  \end{center}

  \vskip.2cm 
  
The CAH cone, originally introduced in the context of Tomita–Takesaki modular theory, plays a fundamental role in quantum geometry and the noncommutative extension of information geometry. In particular, this canonical self-dual cone provides a natural setting for the statistical description of quantum states and their transformations.

Thus, the covariance matrices emerging from the Boltzmann learning process can be rigorously interpreted within the noncommutative geometric framework, where they form a natural subspace of the quantum statistical manifold described by the Connes–Araki–Haagerup cone.

\subsection{Implications}
Theorem 1 rigorously establishes that the set of covariance matrices arising from generative learning with Gaussian latent distributions naturally forms a geometric cone.
In deep learning and Boltzmann machines, this suggests that the optimization of covariance structures happens over a well-defined convex Wishart space.
This reinforces connections with optimal transport since the generative learning process aligns with Monge–Ampère-like structures over the Wishart cone.

Concerning Theorem B, the Connes–Araki–Haagerup cone appears in Tomita–\break Takesaki modular theory, which governs the dynamics of von Neumann algebras and noncommutative probability. If the space of covariance matrices in machine learning (or generative models) forms a CAH cone, then:

\begin{itemize}
\item The covariance matrices naturally inherit a modular flow structure.

\item The learning process (e.g., in Boltzmann machines) could be described in terms of modular automorphisms, meaning there exists a deeper algebraic structure underlying information propagation.
\end{itemize}
This could lead to new optimization methods that take advantage of mass-conserving flows to improve training efficiency.
\subsection{Plan of the paper}
The paper is structured as follows:
\begin{itemize}
    \item Section 2: Boltzmann Machines and Energy-Based Models

In this section, we introduce the theoretical foundations of Boltzmann machines and more general energy-based models. We review their probabilistic structure, particularly the role of free energy minimization in learning, and highlight how these models attempt to represent complex probability distributions. This section sets the stage for understanding how optimal transport methods, and specifically the Monge–Ampère equation, emerge in deep learning.

\item[]

\item Section 3: The Monge–Ampère Equation in Deep Learning and Boltzmann Models

Here, we establish the connection between deep learning and the Monge–\break Ampère equation. We explain how learning in Boltzmann machines and related deep generative models can be formulated in terms of optimal transport theory, where the Monge–Ampère equation governs the transformation of probability distributions. This section also integrates insights from quantum geometry, showing how certain structures in deep learning models can be interpreted through the lens of noncommutative geometry, in particular the Connes–Araki–Haagerup (CAH) cone and related mathematical frameworks.
\item[]
\item Section 4: An Alternative Approach via Renormalization Group (RG) Flow

While Section 3 develops a connection between the Monge–Ampère equation and learning, Section 4 presents a very different perspective based on renormalization group (RG) flow. This approach, though distinct from the optimal transport formulation, provides a more concrete and computationally feasible way of describing the hierarchical structure of learning. We argue that this method offers a different manifestation of the existence of a Monge–Ampère domain within the learning process, revealing how the equation's influence can be understood from multiple perspectives.

\item[]

\item Section 5: Conclusion and Future Directions

Finally, in the conclusion, we summarize the key insights and results from our discussion, emphasizing the dual perspectives (optimal transport vs. RG flow) on how the Monge–Ampère equation influences deep learning. We also propose possible future research directions, including:
\end{itemize}

In this paper, we have therefore shown in an innovative way the critical role of the Monge–Ampère equation in deep learning, particularly in Boltzmann models and generative learning.

Via our method, we have also provided quantum geometry insights into the structure of deep learning, including connections to von Neumann algebras and the CAH cone.

\vskip.2cm 

{{\bf Acknowledgements}
This research is part of the project No. 2022/47/P/ST1/01177 cofunded by the National Science Centre  and the European Union's Horizon 2020 research and innovation program, under the Marie Sklodowska Curie grant agreement No. 945339 \includegraphics[width=1cm, height=0.5cm]{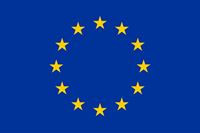}. For the
purpose of Open Access, the author has applied a CC-BY public copyright license
to any Author Accepted Manuscript (AAM) version arising from this submission.
}

\section{Boltzmann Machines and Energy-Based Models}
\subsection{Boltzmann Machines as Energy-Based Models}
We first start by recalling a few notions on Boltzmann machines and describe how Boltzmann machines represent probability distributions via an energy function. 

\subsubsection{Recollections}
A Boltzmann machine is a network of stochastic $N$ neurons, which consists of:
\begin{itemize}
    \item Visible units (stochastic neurons)
    \item Hidden units
    \item Weighted connections that define an energy-based probability distribution.
\end{itemize}

The network is trained by adjusting weights so that the probability distribution over visible units matches the data distribution. The training process involves sampling via Markov Chain Monte Carlo (MCMC) methods like Gibbs sampling.

\subsubsection{Energy-based models}
 Boltzmann Machines are energy-based models that define a probability distribution over a set of variables using an energy function 
$E({\bf x};\theta)$, where 
${\bf x}$ is the state of the system and 
$\theta$ are the model parameters.

The probability of a state 
${\bf x}$ is given by the Boltzmann distribution:
\[p({\bf x};\theta)= \frac{1}{Z(\theta)}
\operatorname{exp}(-E(\textbf{x};\theta)),\]

where $Z(\theta)$ is the partition function. Training a Boltzmann Machine involves minimizing the energy of observed data configurations, which is related to optimizing a loss function in deep learning \cite{HS}.

\section{The Monge–Ampère Equation in Deep Learning and Boltzmann Models}

We investigate whether the evolution or deformation of these energy landscapes during learning can be understood via PDEs. 

In models such as Wasserstein GANs (see, e.g., Arjovsky et al., 2017 \cite{A}) and normalizing flows, one explicitly designs architectures so that the learned map transports a latent (simple) distribution to the data distribution. Here, the network approximates the optimal transport map whose potential function is analogous to the free energy landscape.

\vskip.2cm

When one uses an optimal transport loss (e.g., the Wasserstein distance) during training, the geometry of the probability measures is taken into account. The learned free energy or potential function is forced to reflect the underlying geometric structure—much as the Monge–Ampère equation constrains the curvature of the optimal transport potential.

\subsection{Monge–Ampère Equations in Deep Learning}

 \subsubsection{Monge–Ampère equation for probability distributions}
The Monge–Ampère equation is central in optimal transport theory, which has already found applications in generative models (e.g., Wasserstein GANs). 

\,

This is a mathematical framework for finding the most efficient way to transport mass from one distribution to another, given a cost function.

\, 

In the context of optimal transport, the Monge-Ampère equation arises as a nonlinear partial differential equation (PDE) that characterizes the optimal transport map between two probability distributions.

\, 

If $\rho_0$ and $\rho_1$ are two probability densities, the optimal transport map $T$ satisfies: 
\[\rho_0({\bf x})=\rho_1(T({\bf x}))\cdot\det(DT({\bf x})),\]
where $DT$ is the Jacobian of $T$, \cite{V}.

\,

Both Boltzmann Machines and the Monge-Ampère equation involve energy minimization. In Boltzmann Machines, the energy function $E({\bf x};\theta)$ defines the model, while in optimal transport, the energy is the cost of transporting mass.

\vskip.2cm

The Monge-Ampère equation can be seen as a constraint on the transformation 
$T$ that preserves probability mass, similar to how the partition function 
$Z(\theta)$ in Boltzmann Machines ensures normalization of the probability distribution (see \cite{Cu} for more).

\, 

\subsubsection{Gradient Flows and Wasserstein Distance}
In the context of probability distributions,  the Wasserstein distance, is a metric derived from optimal transport that measures the distance between two probability distributions. It is closely related to the Monge-Ampère equation.

\vskip.2cm 

Training Boltzmann Machines can be viewed as a gradient flow in the space of probability distributions, where the goal is to minimize the energy functional. This is analogous to the gradient flow structure in optimal transport, where the Wasserstein distance is minimized, \cite{JK}.

In the context of Boltzmann machines and generative models, the latent distribution and the data distribution play crucial roles in understanding how the model learns to represent data.

\subsubsection{Definition of the Data Distribution $P_{\text{data}}(x)$}
The data distribution $P_{\text{data}}(x)$ is the probability distribution over the observed (or visible) data points. Formally, given a dataset $\{ x_i \}_{i=1}^{N}$ consisting of samples in some space $X$, the empirical data distribution is given by:
\begin{equation}
P_{\text{data}}(x) = \frac{1}{N} \sum_{i=1}^{N} \delta(x - x_i),
\end{equation}
where $\delta(x - x_i)$ is the Dirac delta function centered at the observed data point $x_i$. In practice, we assume that the data follows an unknown probability distribution, which the model aims to approximate.

\subsubsection{Definition of the Latent Distribution $P_{\text{latent}}(z)$}
The latent distribution $P_{\text{latent}}(z)$ is an auxiliary probability distribution over a lower-dimensional space $Z$, which serves as a simpler representation of the data. This distribution is typically chosen to have a well-defined, tractable form, such as:
\begin{equation}
P_{\text{latent}}(z) = \mathcal{N}(0, I),
\end{equation}
where $\mathcal{N}(0, I)$ is a standard Gaussian distribution in $\mathbb{R}^d$, meaning the components of $z$ are independently normally distributed with mean 0 and variance 1.

The role of the latent distribution is to provide a compact and structured representation of the data, from which samples can be drawn and mapped back to the data space via a transformation.

\subsubsection{Relationship Between the Latent and Data Distributions}
To generate samples that resemble real data, a generative model must learn a mapping $T: Z \to X$ that transforms samples from the latent distribution into the data space. The model effectively learns a push-forward measure, where the transformed distribution $P_{\theta}(x)$ of generated samples should approximate $P_{\text{data}}(x)$. This means:
\begin{equation}
P_{\theta}(x) = T_{\#} P_{\text{latent}}(z),
\end{equation}
where $T_{\#} P_{\text{latent}}$ denotes the push-forward measure under $T$, meaning:
\begin{equation}
P_{\theta}(x) = P_{\text{latent}}(T^{-1}(x)) \left| \det \nabla T^{-1}(x) \right|.
\end{equation}
This is closely related to the Monge–Ampère equation in optimal transport, which describes how probability mass is transported under a convex function.

\subsubsection{The Role in Boltzmann Machines and Energy-Based Models}
In Boltzmann machines, the data distribution is modeled as:
\begin{equation}
P_{\theta}(x) = \frac{1}{Z} \sum_{h} e^{-E_{\theta}(x,h)},
\end{equation}
where $h$ are latent (hidden) variables and $E_{\theta}(x,h)$ is an energy function. The goal is to learn parameters $\theta$ such that $P_{\theta}(x) \approx P_{\text{data}}(x)$.

The latent distribution in this case is implicitly defined over the hidden units $h$, and the model learns a transformation from the hidden units to the visible data.

\subsubsection{Connection to the Monge–Ampère Equation}
The process of mapping from $P_{\text{latent}}(z)$ to $P_{\text{data}}(x)$ can be viewed as an optimal transport problem, where one seeks a transformation $T$ satisfying:
\begin{equation}
\det (D^2 T(x)) = \frac{P_{\text{data}}(x)}{P_{\text{latent}}(T^{-1}(x))}.
\end{equation}
This is precisely the Monge–Ampère equation, which characterizes optimal mass transport. The interpretation is that training a Boltzmann machine or another energy-based generative model is equivalent to solving an implicit Monge–Ampère equation that governs the transport of mass from the latent space to the data space.
Therefore, 
\begin{itemize}
    \item the data distribution $P_{\text{data}}(x)$ is the empirical or real-world distribution of observed samples.
    \item The latent distribution $P_{\text{latent}}(z)$ is an auxiliary, usually simple distribution from which we sample and map to the data space.
    \item The learning process in a Boltzmann machine (or other generative models) involves finding a transport map from $P_{\text{latent}}$ to $P_{\text{data}}$, which can be described by the Monge–Ampère equation in optimal transport.
\end{itemize}
This interpretation provides a geometric and optimal transport perspective on generative models, linking them directly to the theory of Monge–Ampère equations.

\, 

\subsection{Wishart matrices and Wishart cones}

There is a deep connection between Wishart matrices, Wishart cones, and the structure of latent distributions in Boltzmann machines. This connection arises through random matrix theory, optimal transport, and geometric probability theory.

\subsubsection{Wishart Matrices and Latent Distributions}
A Wishart matrix is a random positive semidefinite matrix of the form:
\begin{equation}
W = X X^T
\end{equation}
where $X$ is an $n \times m$ matrix whose entries are independent Gaussian random variables. When 
\begin{equation}
X \sim \mathcal{N}(0, I),
\end{equation}
then $W$ follows the Wishart distribution, denoted:
\begin{equation}
W \sim W_n(m, I).
\end{equation}
This means that the latent variable $Z$ in generative models, which is often sampled from $\mathcal{N}(0, I)$, naturally induces a Wishart-distributed covariance structure in transformed spaces.

\subsubsection{Wishart Cones and Probability Measures}
The Wishart cone consists of all positive semidefinite matrices of fixed rank. It is a homogeneous space that can be interpreted as a convex domain where certain structured probability distributions live.

\begin{itemize}
    \item The latent distribution $P_{\text{latent}}(z) = \mathcal{N}(0, I)$ induces a Wishart-distributed covariance structure when mapped through a neural network.
    \item The data distribution $P_{\text{data}}(x)$ can thus be modeled using a push-forward of this Wishart-like structure.
    \item This aligns with optimal transport theory, where the Monge–Ampère equation governs mass transport between distributions. In the case of latent-variable generative models, the transport map typically transforms a Wishart-type structure into the observed data space.
\end{itemize}

\subsubsection{Connection to Boltzmann Machines and Energy-Based Models}
For Boltzmann machines, the learned weight matrix $W$ (connecting hidden and visible units) is often correlated with the covariance structure of the data. If the data follows a Gaussian-like structure, then:
\begin{equation}
W = X X^T
\end{equation}
where $X$ consists of transformed latent variables. This suggests a Wishart-like nature of the learned weights.

\begin{itemize}
    \item The latent distribution is Gaussian: $Z \sim \mathcal{N}(0, I)$.
    \item The covariance of the transformed distribution follows a Wishart structure.
    \item The learned transport map corresponds to a solution of a Monge–Ampère-like equation, ensuring mass conservation.
\end{itemize}

This reinforces the geometric structure of the Boltzmann machine learning process and provides a new perspective: the generative learning process can be seen as an optimization over the space of Wishart matrices.

\subsubsection{Implication: Geometry and Learning}
The Wishart structure also suggests a natural geometric perspective for deep learning models:

\begin{itemize}
    \item \textbf{Latent-space sampling:} The latent distribution is Gaussian, inducing a Wishart structure.
    \item \textbf{Optimization:} The training process adjusts a transport map (potentially satisfying Monge–Ampère-like equations).
    \item \textbf{Generative reconstruction:} The learned data distribution inherits a covariance structure that aligns with a Wishart cone.
\end{itemize}

This approach connects random matrix theory, optimal transport, and deep learning through the natural emergence of Wishart distributions in neural network training.

Therefore, we can say that 
\begin{itemize}
    \item the latent distribution $P_{\text{latent}}(z) = \mathcal{N}(0, I)$ naturally leads to a Wishart-type structure in covariance matrices.
    \item Boltzmann machines learn weight matrices that inherit Wishart properties, shaping how the model represents data distributions.
    \item The transport map learned by the network is closely related to the Monge–Ampère equation, which governs mass transport in structured probabilistic spaces.
    \item The Wishart cone provides a geometric setting for understanding the representation of covariance matrices in generative models.
\end{itemize}

In particular, we can state that:

\begin{thm}
Let $Z\sim \mathcal{N}(0,I_m)$ be a latent Gaussian distribution in $\R^{m}$, and let $T:\R^m \to \R^n$ be a linear transformation parametrized by a matrix $W\in \R^{n\times m}$. Define the transformed random variable: \[X=WZ.\]
Then, the covariance matrix of $X$ follows a Wishart distribution: $$\Sigma=WW^T\sim  \mathcal{N}_n(m,I).$$
    Furthermore, the space of all such covariance matrices, forms a convex symmetric cone, called the Wishart cone,  consisting of all positive semidefinite matrices of rank at most $min(n,m)$.  This cone is a Monge-Amp\`ere domain.
\end{thm}
This perspective offers a new way to study energy-based models, Boltzmann machines, and neural network training using the geometry of Wishart matrices and optimal transport theory.

\begin{proof}
    Since 
\[
Z \sim N(0, I_m),
\]
the covariance of \( X = WZ \) is:

\[
E[XX^T] = E[(WZ)(WZ)^T] = W E[ZZ^T] W^T.
\]

Since 
\[
E[ZZ^T] = I_m,
\]
we obtain:

\[
\Sigma = WW^T.
\]

By definition, if \( W \) is an \( n \times m \) matrix with \( m \) independent Gaussian columns, then \( WW^T \) follows a Wishart distribution:

\[
\Sigma = WW^T \sim W_n(m, I).
\]

We study below the convex cone structure.

The set of all possible \( \Sigma = WW^T \) satisfies the following properties:

If \( \Sigma \) is a covariance matrix in the space, then for any \( \lambda > 0 \),

\[
\lambda \Sigma = (\lambda W)(\lambda W)^T
\]

also belongs to the space, since it is still of the form \( W' W'^T \) for some \( W' \).

If 

\[
\Sigma_1 = W_1 W_1^T \quad \text{and} \quad \Sigma_2 = W_2 W_2^T
\]

are two elements in the space, their sum

\[
\Sigma_1 + \Sigma_2 = W_1 W_1^T + W_2 W_2^T
\]

is also a sum of two positive semidefinite matrices and hence remains positive semidefinite, ensuring the space is closed under addition.

Every element of the space is positive semidefinite:

\[
v^T \Sigma v = v^T WW^T v = (W^T v)^T (W^T v) \geq 0, \quad \forall v \in \mathbb{R}^n.
\]

Furthermore, since \( W \) has at most rank \( \min(n, m) \), so does \( \Sigma \), confirming that the space is restricted to a cone of positive semidefinite matrices of at most this rank.

The space of all covariance matrices induced by a latent Gaussian \( Z \sim N(0, I) \) through a linear transformation \( W \) forms a convex cone of Wishart-distributed positive semidefinite matrices, commonly known as the {Wishart cone}.

\,

Finally, From the works in \cite{C0,C2,C3} it follows that the cone is a Monge--Ampère domain i.e. everywhere locally a Monge--Ampère equation is satisfied on the interior of the Wishart cone. 
\end{proof}
\begin{thm}
The space of all covariance matrices 
\[
\Sigma = WW^T
\]
arising in the Boltzmann learning process, where \( W \) is an \( n \times m \) matrix with independent Gaussian columns, forms a convex cone of positive semidefinite matrices. This space is naturally endowed with a \textit{Wishart structure}, as each covariance matrix follows a Wishart distribution:

\[
\Sigma = WW^T \sim W_n(m, I).
\]

Furthermore, this space coincides with the \textit{Connes–Araki–Haagerup (CAH) cone} associated with a finite-dimensional von Neumann algebra.
\end{thm}
\begin{proof}
    This is a consequence of the above statement and of the result in \cite{C1}.
\end{proof}



\, 

\section{An Alternative Approach via Renormalization Group (RG) Flow}
We discuss the Monge-Ampère aspect of the learning methods for relative Boltzmann machines from a different appraoch. In the renormalization group (RG) framework, one performs a sequence of coarse-graining transformations that map fine-scale (microscopic) descriptions to coarse-scale (macroscopic) ones. If one views the BM’s hidden layers as encoding a coarse-grained representation, then the transformation from the latent space to the data space (via the free energy) is conceptually similar to an optimal transport map between distributions at different scales.

\vskip.2cm

In a Boltzmann machine, learning involves finding an energy function \( E \) for visible units and hidden units such that the marginal over the hidden units models the data distribution.

We highlight that when hidden units are interpreted as a \textit{coarse-grained} representation of the data, the process of learning can be seen as analogous to  \textit{coarse-graining}. That is, the hidden layer averages over fine details present in the visible layer. A coarse-graining interpretation of RBMs has been rigorously discussed in the literature, notably in the work of Mehta and Schwab \cite{MS}.

\,

 Coarse-graining is mathematically the process of mapping a high-dimensional ({\it microscopic}) space \( X \) which here will be a fine-scale probability distribution onto a lower-dimensional ({\it macroscopic}) space \( Y \), which will be a corase-frained distribution via a surjective, measurable map \( \pi \): \[\pi:X\to Y.\] 
 The map $\pi$ maps a (micro)state $x\in X$ to a (macro)state $y\in Y$ that captures the essential, large-scale features of $x$.

 The push-forward measure \( \nu \) on \( Y \) aggregates the probabilities of the microstates in \( X \) that are identified as equivalent. In particular, we have: \[\nu(B)=\mu(\pi^{-1}(B))\]

\,

A given macrostate $y$ corresponds to a certain number of microstates. This number is called the multiplicity.
The entropy 
$\mathscr{S}$ of a macrostate is related to the multiplicity $m$ (i.e. the number of microstates, corresponding to the macrostate) by Boltzmann’s formula:

\[\mathscr{S}=k_B\ln m,\] where $k_B$ is the Boltzmann constant.

 \, 
 
In the context of renormalization or coarse-graining, one often defines an effective Hamiltonian. An effective Hamiltonian \( H_{\text{eff}} \) is defined on the space $Y$ so that the coarse-grained partition function over $Y$,  \[Z_{\text{eff}}=\int_Y \operatorname{e}^{-H_{eff}(y)}d\nu(y),\]  

approximates the original partition function \( Z \).

\, 

 Iterating this process gives rise to the RG flow, which is fundamental in understanding scale invariance and universality.

\, 

 In Relative Boltzmann machines, hidden units can be seen as encoding a coarse-grained version of the visible data.

The mapping $\pi$ can be viewed as an optimal transport problem, where one “transports” the probability mass in the microscopic space into the coarse-grained space while preserving key features (or "mass") of the distribution. 

In such contexts, one expects the existence of a potential function $\phi$ whose gradient realizes the transformation, with the Monge--Ampere equation equation emerges as a mass conservation condition \cite{B}:

\[\det(D^2\phi(x))=\frac{\mu(x)}{\nu(\nabla\phi(x))}.\]

\, 

The free energy in energy-based models (such as Boltzmann machines) can be interpreted as a potential function whose gradient is analogous to an optimal transport map. In some formulations, learning in Boltzmann machines can be seen as finding a transformation that minimizes a Wasserstein-type loss between the model distribution and the data distribution.

\, 

The modern approach of using optimal transport ideas in neural networks is not only computationally advantageous (providing robust loss functions and stable training) but also conceptually analogous to Renormalization Group coarse-graining.

\,

It suggests that the learning process in deep networks might be understood as an optimal transport flow, transforming a simple latent measure into a complex, high-dimensional data distribution in a way that mirrors how physical systems are renormalized.

\subsection{Example}
To describe our example, let us consider the two-dimensional Ising model on a lattice where each site carries a spin $s_i\in\{-1,+1\}$.

The Hamiltonian of the Ising model is given by:
\begin{equation}
H(s) = -J \sum_{\langle i,j \rangle} s_i s_j,
\end{equation}
where $\langle i,j \rangle$ indicates neighboring sites, and $J > 0$ is the coupling constant.  

The probability distribution of a configuration ${\bf s}$ (the microstate) at inverse temperature $\beta$ is
\begin{equation}
P({\bf s}) = \frac{1}{Z} e^{-\beta H({\bf s})},
\end{equation}
with the partition function
\begin{equation}
Z = \sum_{{\bf s}} e^{-\beta H(s)}.
\end{equation}

One can use Monte Carlo methods (e.g., the Metropolis algorithm) to sample configurations from this distribution. In our context, these configurations form the visible data for the Restricted Boltzmann Machine (RBM). Each configuration on a lattice of size $L \times L$ is represented as a vector
\begin{equation}
v \in \{-1, +1\}^{L^2}.
\end{equation}

\, 

Now, an RBM is a two-layer energy-based model consisting of:
\begin{itemize}
    \item A visible layer ${\bf v} \in \{-1, +1\}^{L^2}$ (representing the spins of the Ising model),
    \item A hidden layer ${\bf h} \in \{0,1\}^{n_h}$ with $n_h < L^2$.
\end{itemize}

The joint energy function is given by:
\begin{equation}
E(v, h) = - \sum_{i=1}^{L^2} a_i v_i - \sum_{j=1}^{n_h} b_j h_j - \sum_{i,j} v_i W_{ij} h_j,
\end{equation}
where $W_{ij}$ are the weights between visible and hidden units, and $a_i$ and $b_j$ are the visible and hidden biases, respectively.

\,

The marginal probability of a visible configuration is then given by:
\begin{equation}
P(v) = \frac{1}{Z} \sum_{h} e^{-E(v,h)},
\end{equation}
and the associated free energy is:
\begin{equation}
F(v) = - \sum_i a_i v_i - \sum_j \log(1 + \exp(b_j + \sum_i W_{ij} v_i)).
\end{equation}

We illustrate the coarse-graining procedure via the hidden layer. 

Because the number of hidden units $n_h$ is chosen to be significantly smaller than $L^2$, the RBM must learn a compressed (or coarse-grained) representation of the data. The process follows these steps:

\vskip.2cm

1. First we need to map microstates to macrostates.
The visible configuration ${\bf v}$ (a microstate) is mapped to a hidden representation ${\bf h}$. Each hidden unit can be thought of as detecting a local pattern or feature (e.g., the predominant spin orientation in a block of sites).

\vskip.2cm

2. When the RBM is trained, the weight matrix $W$ tends to organize so that each hidden unit responds to a group of nearby spins. For example, on a $4 \times 4$ lattice, if we use $n_h = 4$ hidden units, each hidden unit might effectively summarize the state of a $2 \times 2$ block of spins. This is analogous to block spin renormalization where many spins (microstates) are grouped into a single coarse variable (macrostate).

\vskip.2cm

3. We turn our attention to free energy. 
The free energy $F({\bf v})$ computed by the RBM encodes the likelihood of a configuration. Because the hidden layer integrates over many microscopic configurations (via summation over ${\bf h}$), the free energy provides a coarse-grained description that emphasizes large-scale, collective behavior rather than individual spin flips.

In particular, after training, one often finds that the learned weights have localized, block-like structures. A hidden unit might have strong weights for visible units corresponding to a particular $2 \times 2$ block in the lattice, effectively representing the "average" or majority state of that block—a coarse-grained variable.

Concerning the effective Hamiltonian, 
the free energy $F(v)$ can be seen as an effective Hamiltonian $H_{\text{eff}}(v)$ for the coarse-grained variables. In renormalization group (RG) language, integrating out the hidden variables (or summing over $h$) is analogous to obtaining an effective theory for the macroscopic degrees of freedom.

\, 

The description above coincides with Wilson’s RG, where one defines a transformation $\pi: X \to Y$ that projects a fine-scale configuration $x \in X$ (microstate) to a coarse-scale configuration $y \in Y$ (macrostate). The effective Hamiltonian $H_{\text{eff}}(y)$ is  defined by:
\begin{equation}
e^{-H_{\text{eff}}(y)} = \sum_{x: \pi(x) = y} e^{-H(x)}.
\end{equation}

In the RBM, the marginal probability $P({\bf v})$ is obtained by summing over the hidden units:
\begin{equation}
e^{-F({\bf v})} = \sum_{\bf h} e^{-E({\bf v},{\bf h})}.
\end{equation}
This is directly analogous to integrating out the microscopic degrees of freedom, yielding a coarse-grained description encoded by $F({\bf v})$.

So, to conclude in this approach we have the following important steps: 
\begin{itemize}
    \item The \textbf{microscopic level} coincides with the visible layer of the RBM and it represents the full microscopic spin configuration of the Ising model.
    \item The  \textbf{coarse-graining} is constructing using the hidden layer.  Having fewer units than the visible layer, it forces the model to learn a compressed, coarse-grained representation of the data.
    \item The free energy $F(v)$ acts as an effective Hamiltonian for the coarse-grained description.
\end{itemize}

In essence, training an RBM on Ising model configurations provides a concrete example of coarse-graining in a learning system—where the hidden layer captures large-scale features analogous to block spin variables in renormalization group theory.

\section{Conclusion and Future Directions}
By considering these connections, we  contribute to a deeper theoretical understanding of deep learning. This work could potentially bridge gaps between statistical mechanics, geometric analysis, quantum geometry and modern machine learning—opening up new paths for both theoretical exploration and practical algorithm design.

\vskip.2cm

This interdisciplinary approach not only enriches the theoretical landscape but  paves the way for novel techniques in training and understanding complex models.

\begin{itemize}
    \item We have considered extending the quantum geometric approach to broader class of deep generative models.

  \item We have investigated the computational implications of Monge–Ampère-based learning.

  \item We have furthermore, explored further connections between renormalization, optimal transport, and learning dynamics.
\end{itemize}

In Section 4, we have considered an alternative RG-based approach, showing another perspective on the Monge–Ampère domain in learning, bridging ideas from statistical mechanics, renormalization, and deep learning to present a unified mathematical framework for hierarchical feature learning.

\, 

This paper thus contributes to the mathematical foundations of deep learning, offering both theoretical insights and practical implications for understanding learning as an optimal transport problem and a renormalization process.

\,

Our work suggests possible experiments, where the theoretical predictions (based on PDE analysis) are tested against the empirical behavior of deep learning models. For instance, one could investigate whether introducing regularization terms inspired by these equations (Monge--Ampere) leads to more robust or interpretable models.


\begin{thebibliography}{9}

 \bibitem{A} Arjovsky, M; Chintala, S. $\&$ Bottou, L. {\sl Wasserstein Generative Adversarial Networks}
 Proceedings of the 34th International Conference on Machine Learning, PMLR 70:214-223, 2017.
 
 \bibitem{B} Brenier, Y. {\sl Polar factorization and monotone rearrangement of vector-valued functions,} Communications on Pure and Applied Mathematics 44 (1991): 375–417.

\bibitem{C0} Combe, N. C. {\sl On Frobenius structures in symmetric cones,} arXiv:2309.04334 


 \bibitem{C1} Combe, N. C. {\sl Wishart cones and quantum geometry}, to appear in Lecture Notes in Networks and Systems (2025).

\bibitem{C2} Combe, N. C.  
{\sl Learning on hexagonal structures and Monge-Ampère operators} arXiv:2412.04407 


\bibitem{C3}
Combe, N. C. {\sl Maximum Likelihood, permutohedra and Associativity Equations} arXiv:2501.01345    

 \bibitem{CD} Cavalli, A. $\&$ Decherchi, S.  {\sl Optimal Transport for Free Energy Estimation} J. Phys. Chem. Lett. 2023, 14, 6, 1618–1625.
 
 \bibitem{Cu} Cuturi, M. (2013). {\sl Sinkhorn Distances: Lightspeed Computation of Optimal Transport.} Advances in Neural Information Processing Systems (NeurIPS).


 \bibitem{DMH}Daniels, D., Maunu, T. $\&$ Hand, P. {\sl Score-based Generative Neural Networks for Large-Scale Optimal Transport} 35th Conference on Neural Information Processing Systems (NeurIPS 2021).


  \bibitem{JK} Jordan, R., Kinderlehrer, D., $\&$ Otto, F. (1998). {\sl The Variational Formulation of the Fokker-Planck Equation. } SIAM Journal on Mathematical Analysis.


  \bibitem{MS}  {Mehta, P., $\&$ Schwab, D. J.} (2014). {\sl An exact mapping between the Variational Renormalization Group and Deep Learning.} Proceedings of the 2nd International Conference on Learning Representations (ICLR). 

 \bibitem{HS} Hinton, G. E., $\&$ Sejnowski, T. J. (1986).{\sl Learning and Relearning in Boltzmann Machines.} In Parallel Distributed Processing: Explorations in the Microstructure of Cognition (Vol. 1).

\bibitem{V} Villani, C. (2009). {\sl Optimal Transport: Old and New.} Springer.
\end{thebibliography}
\end{document}